\documentclass{colt2012} 

\usepackage{amsmath, amssymb, multirow, paralist}
\usepackage{algorithm,algorithmic}
\usepackage{color}

\usepackage{hyperref}
\hypersetup{
colorlinks=true,
citecolor=blue,
linkcolor = {red},
linkbordercolor={1 1 1}, 
citebordercolor={1 1 1} 
}

\def \x {\mathbf{x}}
\def \e {\mathbf{e}}

\def \z {\mathbf{z}}
\def \f {\mathbf{f}}
\def \P {\mathcal{P}}
\def \R {\mathbb{R}}
\def \u {\mathbf{u}}

\def \E {\mathrm{E}}

\def \R {\mathbb{R}}
\def \regret {\mbox{regret}}
\def \f {\mathbf{f}}
\def \VAR {\mbox{VAR}}

\begin{document}

\title[Regret Bound by Variation]{Regret Bound by Variation for Online Convex Optimization}

 \coltauthor{\Name{Tianbao Yang$^\dagger$} \Email{yangtia1@msu.edu}\\
\Name{Mehrdad Mahdavi$^\dagger$} \Email{mahdavim@cse.msu.edu}\\
\Name{Rong Jin$^\dagger$} \Email{rongjin@cse.msu.edu}\\
\Name{Shenghuo Zhu$^*$} \Email{zsh.@sv.nec-labs.com}\\
       \addr $^\dagger$Department of Computer Science and Engineering\\
       Michigan State University, East Lansing, MI  48824, USA\\
       \addr $^*$NEC Laboratories America, Cupertino, CA 95014, USA
 }


\maketitle

\begin{abstract}

In~\citep{Hazan-2008-extract}, the authors showed that the regret of the Follow the Regularized Leader (FTRL) algorithm for online linear optimization can be bounded by the total variation of the cost vectors. In this paper, we extend this result to general online convex optimization. We first analyze the limitations of  the FTRL algorithm in~\citep{Hazan-2008-extract} when applied  to online convex optimization, and extend the definition of variation to a sequential variation which is shown to be a lower bound of the total variation. We then present two novel algorithms that bound the regret by the sequential variation of cost functions. Unlike previous approaches that maintain a single sequence of solutions, the proposed algorithms  maintain two sequences of solutions that makes it possible to achieve a variation-based regret bound for online convex optimization. 

\end{abstract}

\section{Introduction}

We consider the general online convex optimization problem~\citep{DBLP:conf/icml/Zinkevich03} which proceeds in trials. At each trial, the learner is asked to predict the decision vector $\x_t$ that belongs to a bounded closed convex set $\P\subseteq\mathbb R^d$; it then receives a cost function $c_t(\cdot):\P\rightarrow \R$ and incurs a cost of $c_t(\x_t)$. The goal of online convex optimization is to come up with a sequence of solutions $\x_1, \ldots, \x_T$ that minimizes the regret, which is defined as the difference in the cost of the sequence of decisions accumulated up to the trial $T$ made by the learner and the cost of the best fixed decision in hindsight, i.e.
\begin{align*}
\regret = \sum_{t=1}^T c_t(\x_t) -\min_{\x\in\P}\sum_{t=1}^T c_t(\x).
\end{align*}
In a special case, when the cost functions are linear $c_t(\x) = \f_t^{\top}\x$, the problem becomes the online linear optimization. The goal of online convex optimization is to design algorithms that predict, with a small regret, the solution $\x_t$ at the $t$th trial given the (partial) knowledge about the past cost functions $c_\tau(\cdot), \tau=1,\cdots, t-1$. Many algorithms have been proposed for online convex optimization, especially for online linear optimization.  \citet{DBLP:conf/icml/Zinkevich03} proposed a gradient descent algorithm for online convex optimization with a regret bound of $O(\sqrt{T})$. When cost functions are strongly convex, the regret bound of the online gradient descent algorithm is reduced to $O(\log(T))$ with appropriately chosen step size~\citep{Hazan:2007:LRA:1296038.1296051}, and to $O(1)$ by a more recent work~\citep{DBLP:journals/jmlr/HazanK11a}. Another common methodology for online convex optimization, especially for online linear optimization, is based on the framework of Follow the Leader (FTL)~\citep{Kalai:2005:EAO:1113185.1113189}. FTL chooses $\x_t$ by minimizing the cost incurred by $\x_t$ in all previous trials. Since the naive FTL algorithm fails to achieve a sublinear regret in the worst case, many variants have been developed to fix the problem, including Follow the Perturbed Leader (FTPL)~\citep{Kalai:2005:EAO:1113185.1113189}, Follow the Regularized Leader (FTRL)~\citep{DBLP:conf/colt/AbernethyHR08}, and Follow the Approximate Leader (FTAL)~\citep{Hazan:2007:LRA:1296038.1296051}. Other methodologies for online convex optimization introduce a potential function (or link function) to maps solutions between the space of primal variables and the space of dual variables, and carry out primal-dual update based on the potential function.  The well-known Exponentiated Gradient (EG) algorithm~\citep{Kivinen:1995:AVE:225058.225121}  or multiplicative weights algorithm~\citep{Freund:1995:DGO:646943.712093} belong to this category. We note that these different algorithms are closely related. For example,  in online linear optimization, the potential-based primal-dual algorithm is equivalent to FTRL algorithm~\citep{Hazan-2008-extract}.  All of these studies bound the regret by the number of trials $T$.

An open problem posed in~\citep{citeulike:2404144} was whether it is possible to derive a regret bound for an online algorithm by the variation of the observed costs. It has been established as a fact that the regret of a natural algorithm in a stochastic setting can be bounded by the total variation in the cost vectors~\citep{DBLP:journals/ml/HazanK10}. Therefore, it is of great interest to derive a variation-based regret bound for online convex optimization in an adversarial setting (vs. stochastic setting). Recently \citep{Hazan-2008-extract, DBLP:journals/ml/HazanK10} made a substantial progress in this route. They proved a variation-based regret bound for online \textit{linear} optimization by the FTRL algorithm with an appropriately chosen step size. A similar regret bound is shown in the same paper for prediction from expert advice by modifying the multiplicative weighted algorithm.  In this work, we aim to take one step further. Our goal is to develop algorithms for online \textit{convex} optimization with variation-based regret bounds. In the remaining of this section, we first present the results from \citep{Hazan-2008-extract,DBLP:journals/ml/HazanK10} for online linear optimization and discuss its potential limitations when applied to online convex optimization.

 \subsection{Online Linear Optimization}

Many decision problems can be cast into online linear optimization problems, such as prediction from expert advice~\citep{Cesa-Bianchi:2006:PLG:1137817}, online shortest path problem~\citep{Takimoto:2003:PKM:945365.964295}. \citep{Hazan-2008-extract, DBLP:journals/ml/HazanK10} proved the first variation-based regret bound for online linear optimization problems in an adversarial setting.  Hazan and Kale's  algorithm for online \textit{linear} optimization is based on the framework of FTRL. For completeness, the algorithm is shown in Algorithm~\ref{alg:0}.  At each trial, the decision vector $\x_t$ is given by solving the following optimization problem:
\begin{align*}
\x_t=\arg\min_{\x\in\P} \sum_{\tau=1}^{t-1}\f_\tau^{\top}\x +  \frac{1}{2\eta}\|\x\|_2^2,
\end{align*}
where $\f_t$ is the cost vector received at trial $t$ after predicting the decision $\x_t$, and $\eta$ is a step size.  They bound the regret by the variation of cost vectors defined as
\begin{eqnarray}
    \VAR_T=\sum_{t=1}^T\|\f_t-\mu\|_2^2,\label{eqn:var-linear}
\end{eqnarray}
where $\mu=1/T\sum_{t=1}^T \f_t$. By assuming $\|\f_t\|_2\leq 1, \forall t$ and setting $\eta=\min(2/\sqrt{\text{VAR}_T}, 1/6)$, they showed that the regret of Algorithm~\ref{alg:0} can be bounded by
\begin{align}\label{eqn:boundvar}
\sum_{t=1}^T \f_t^{\top}\x_t -\min_{\x\in\P}\sum_{t=1}^T \f_t^{\top}\x\leq\left\{\begin{array}{lc} 15\sqrt{\text{VAR}_T}& \text{ if }\sqrt{\text{VAR}}_T\geq 12\\ 150& \text{ if } \sqrt{\text{VAR}}_T\leq 12\end{array}\right..
\end{align}
From (\ref{eqn:boundvar}), we can see that when the variation of the cost vectors is small (less than $12$), the regret is a constant, otherwise it is bounded by the variation $O\left(\sqrt{\text{VAR}}_T\right)$.

\begin{algorithm}[t]
\center \caption{Follow The Regularized Leader (FTRL) for Online
Linear Optimization}
\begin{algorithmic}[1] \label{alg:0}
    \STATE {\bf Input}: $\eta>0$

    \FOR{$t = 1, \ldots, T$}
        \STATE  If $t=1$, predict $\x_t=0$
        \STATE  If $t>1$, predict $\x_t$ by $\x_{t} =  \mathop{\arg\min}\limits_{\x \in \P}  \sum_{\tau=1}^{t-1} \f_{\tau}^{\top}\x +\frac{1}{2\eta}\|\x\|_2^2$
        \STATE Receive a cost vector $\f_t$ and incur a loss $\f_t^{\top}\x_t$
    \ENDFOR
\end{algorithmic}
\end{algorithm}

\subsection{Online Convex Optimization}
Online convex optimization generalizes online linear optimization by replacing linear cost functions with non-linear convex cost functions. It has found applications in several domains, including portfolio management~\citep{Agarwal:2006:APM:1143844.1143846}, online classification~\citep{onlinekernellearning}. For example, in online portfolio management problem, an investigator wants to distribute his wealth over a set of stocks without knowing the market output in advance. If we let $\mathbf x_t$ denote the distribution on the stocks and $\mathbf r_t$ denote the price relative vector, i.e. $r_t[i]$ denote the  the ratio of the closing price of stock $i$ on day $t$ to the closing price on day $t-1$,  then an interesting function is the logarithmic growth ratio, i.e. $\sum_{t=1}^T \log(\mathbf x_t^{\top}\mathbf r_t)$, which is a concave function need to be maximized. Similar to~\citep{Hazan-2008-extract,  DBLP:journals/ml/HazanK10}, we aim to develop algorithms for online convex optimization with regrets bounded by the variation in the cost functions. Before presenting our algorithms, below we first show that directly applying the FTRL algorithm to general online convex optimization may not be able to achieve the desirable result.


%

To extend FTRL for online convex optimization, a straightforward approach is to use the first order approximation
for convex cost function, i.e., $c_t(\x)\simeq c_t(\x_t) + \nabla c_t(\x_t)^{\top}(\x-\x_t)$, and replace the cost
vector $\f_t$ in Algorithm~\ref{alg:0} with the gradient of the cost function $c_t(\cdot)$ at $\x_t$, i.e.
$\f_t=\nabla c_t(\x_t)$. Using  the convexity of $c_t(\cdot)$, we have
\begin{align}\label{eqn:boundvar2}
\sum_{t=1}^T c_t(\x_t) -\min_{\x\in\P}\sum_{t=1}^T c_t(\x)\leq \sum_{t=1}^T \f_t^{\top}\x_t -\min_{\x\in\P}\sum_{t=1}^T \f_t^{\top}\x.
\end{align}
If we assume $\|\nabla c_t(\x)\|_2\leq 1, \forall t, \forall \x \in \P$, we can  apply Hazan and Kale's variation-based bound in~(\ref{eqn:boundvar}) to bound the regret in~(\ref{eqn:boundvar2}) by the variation
\begin{align}\label{eqn:var2}
\text{VAR}_{T}=\sum_{t=1}^T\|\f_t -\mu\|_2^2 = \sum_{t=1}^T\left \|\nabla c_t(\x_t) - \frac{1}{T}\sum_{\tau=1}^T \nabla c_\tau(\x_\tau)\right\|_2^2.
\end{align}
To better understand $\VAR_T$ in~(\ref{eqn:var2}), we rewrite $\text{VAR}_T$ as
\begin{align}\label{eqn:var}
\text{VAR}_T & = \sum_{t=1}^T\left \|\nabla c_t(\x_t) - \frac{1}{T}\sum_{\tau=1}^T \nabla c_\tau(\x_\tau)\right\|_2^2 =
\frac{1}{2T} \sum_{t,\tau = 1}^T \|\nabla c_t(\x_t) - \nabla
c_\tau(\x_\tau)\|^2 \notag\\
& \leq  \frac{1}{T} \sum_{t=1}^T \sum_{\tau=1}^T \|\nabla c_t(\x_t) -
\nabla c_t(\x_\tau)\|_2^2 + \frac{1}{T} \sum_{t=1}^T \sum_{\tau=1}^T
\|\nabla c_t(\x_\tau) - \nabla c_\tau(\x_\tau)\|_2^2 \notag\\
& =  \text{VAR}^1_T + \text{VAR}_T^2.
\end{align}
We see that the variation $\VAR_T$ is bounded by two parts: $\VAR^1_T$ essentially measures the smoothness of individual cost functions, while $\VAR^2_T$ measures the variation in the gradients of cost functions. As a result, even when all the cost functions are identical, $\VAR^2_T$ vanishes, while $\VAR^1_T$ still exists, and therefore the regret of the FTRL algorithm for online convex optimization may still be bounded by $O(\sqrt{T})$ regardless of the smoothness of the cost function.

To address this challenge, we develop two novel algorithms for online convex optimization that bound the regret by the variation of cost functions.  In particular, we would like to bound the regret of online convex optimization by the variation of cost functions defined as follows
\begin{eqnarray}
    \VAR^{s}_T = \sum_{t=1}^{T-1} \max\limits_{\x \in \P} \|\nabla c_{t+1}(\x) - \nabla c_{t}(\x)\|_2^2. \label{eqn:var-1}
\end{eqnarray}
Note that the variation in~(\ref{eqn:var-1}) is defined in terms of sequential difference between individual cost function to its previous one, while the variation in~(\ref{eqn:var-linear})~\citep{Hazan-2008-extract} is defined in terms of total difference between individual cost vectors to their mean. Therefore we refer to the variation defined in~(\ref{eqn:var-1})  as \textit{sequential variation}, and to the variation defined in~(\ref{eqn:var-linear}) as \textit{total variation}. It is straightforward to show that when $c_t(\x) = \f_t^{\top}\x$, the sequential variation $\VAR^s_T$ defined in (\ref{eqn:var-1}) is upper bounded  by the total variation $\VAR_T$ defined in (\ref{eqn:var-linear}) with a constant factor:
\begin{align*}
\sum_{t=1}^{T-1}\|\f_{t+1} - \f_t\|_2^2 \leq \sum_{t=1}^{T-1} 2\|\f_{t+1} - \mu \|_2^2 + 2\|\f_t - \mu \|_2^2 \leq 4 \sum_{t=1}^{T}\|\f_t-\mu\|_2^2.
\end{align*}
On the other hand, we can not bound the total variation by the sequential variation up to a constant. This is verified by the following example: $\f_1=\cdots=\f_{T/2} = \f$ and $\f_{T/2+1}=\cdots=\f_T =\mathbf g \neq \f$. The total variation in~(\ref{eqn:var-linear}) is given by
\[
\VAR_T= \sum_{t=1}^T\|\f_t -\mu\|_2^2 = \frac{T}{2}\left\|\f- \frac{\f + \mathbf g}{2}\right\|_2^2 + \frac{T}{2}\left\|\mathbf g - \frac{\f+ \mathbf g}{2}\right\|_2^2  = O(T),
\]
while the sequential variation defined in~(\ref{eqn:var-1}) is a constant given by
\begin{align*}
\VAR^s_T = \sum_{t=1}^{T-1} \|\f_{t+1} - \f_t\| _ 2^2  =  \|\f- \mathbf g\|_2^2  = O(1).
\end{align*}
Based on the above analysis, we claim that the regret bound by sequential variation is usually tighter than by total variation.

The remainder of the paper is organized as follows. We present in section~\ref{sec:algorithm} the proposed algorithms and the main results. In section~\ref{sec:conc}, we conclude this work and discuss how to extend the proposed  algorithms to online bandit convex optimization with  a variation-based regret bound. 

\section{Algorithms and Main Results}
\label{sec:algorithm}
Without loss of generality, we assume the decision set $\P$ is contained in a unit ball $\mathcal B$, i.e., $\P\subseteq\mathcal B$, and $0\in\P$~\citep{Hazan-2008-extract}. We propose two algorithms for online convex optimization. The first algorithm is an improved FTRL and the second one is based on the mirror prox method~\citep{Nemirovski2005}. One common feature shared by the two algorithms is that both of them maintain two sequences of solutions: decision vectors $\x_{1:T}=(\x_1,\cdots,\x_T)$ and searching vectors $\z_{1:T}=(\z_1,\cdots, \z_T)$ that facilitate the updates of decision vectors. Both algorithms share almost the same regret bound except for a constant factor. 
To facilitate the discussion, besides the variation of cost functions defined in (\ref{eqn:var-1}), we define another variation, named \textit{extended sequential variation}, as follows
\begin{align}\label{eqn:vard}
\text{EVAR}^s_T(\z_{1:T})=\sum_{t=0}^{T-1}\|\nabla c_{t+1}(\z_t)-\nabla c_{t}(\z_t)\|_2^2 \leq \|\nabla c_1(\z_0)\|_2^2 + \VAR_T^s,
\end{align}
where $c_0(\x) = 0$ and $\z_0$ is specified in algorithms (usually is zero). When all cost functions are identical, $\text{VAR}^s_T$ becomes zero and the extended variation $\text{EVAR}^s_T(\z_{1:T})$ is reduced to $\|\nabla c_1(\z_0)\|_2^2$, a constant independent from the number of trials. In the sequel, we use the notation $\text{EVAR}^s_T$ for simplicity. In this study, we assume smooth cost functions with Lipschtiz continuous gradients, i.e. there exists a constant $L>0$ such that
\begin{align}\label{eqn:smooth}
\|\nabla c_t(\x)-\nabla c_t(\z)\|_2\leq L\|\x-\z\|_2, \forall \x, \z\in\P, \forall t.
\end{align}
Our results show that for online convex optimization with $L$-smooth cost functions, the regrets of the proposed algorithms can be bounded as follows
\begin{align}\label{eqn:var-main}
\sum_{t=1}^Tc_t(\x_t) -\min_{\x\in\P}\sum_{t=1}^Tc_t(\x)\leq O\left(\sqrt{\text{EVAR}^s_T}\right) + \text{constant}.
\end{align}


\noindent \textbf{Remark:} We would like to emphasize that our assumption about the smoothness of cost functions is necessary to achieve the variation-based bound stated in (\ref{eqn:var-main}). To see this, consider the special case of $c_1(\x) = \cdots = c_T(\x) = c(\x)$. If the bound in~(\ref{eqn:var-main}) holds for any sequence of convex functions, then for the special case where all cost functions are identical, we will have
\begin{align*}
\sum_{t=1}^T c(\x_t) \leq \min_{\x\in\P} \sum_{t=1}^T c(\x)  + O(1),
\end{align*}
implying that $\widehat\x_T = (1/T)\sum_{t=1}^T \x_t$ approaches the optimal solution at the rate of $O(1/T)$. This contradicts the lower complexity bound (i.e. $O(1/\sqrt{T})$) for any first order optimization method~\citep[Theorem 3.2.1]{citeulike:4123765}.

%
%
%

\subsection{An Improved FTRL Algorithm for Online Convex Optimization}\label{sec:IFTRL}

\begin{algorithm}[t]
\center \caption{Improved FTRL  for Online Convex Optimization}
\begin{algorithmic}[1] \label{alg:2}
    \STATE {\bf Input}: $\eta\in(0,1]$

    \STATE {\bf Initialization}: $\z_0 = \mathbf{0}$ and $c_0(\x) =
    0$

    \FOR{$t = 1, \ldots, T$}
        \STATE Predict $\x_t$ by
        $
            \x_t = \mathop{\arg\min}\limits_{\x \in \P} \left\{\x^{\top}\nabla c_{t-1}(\z_{t-1})+
            \frac{L}{2\eta}\|\x - \z_{t-1}\|_2^2
            \right\}
        $
        \STATE Receive a cost function $c_t(\cdot)$ and incur a loss $c_t(\x_t)$
        \STATE Update  $\z_{t}$ by
        $
           \z_{t} = \mathop{\arg\min}\limits_{\x \in \P} \left\{\sum_{\tau=1}^t \nabla c_{\tau}(\z_{\tau-1})^{\top}\x +\frac{L}{2\eta}\|\x\|_2^2\right\}$
    \ENDFOR
\end{algorithmic}
\end{algorithm}

The improved FTRL algorithm for online convex optimization is presented in Algorithm~\ref{alg:2}. Note that in step 6, the searching vectors $\z_t$ are updated according to the FTRL algorithm after receiving the cost function $c_t(\cdot)$. To understand the updating procedure for the decision vector $\x_t$ specified in step 4, we rewrite it as
        \begin{align}
            \x_t
            & = \mathop{\arg\min}\limits_{\x \in \P} \left\{ c_{t-1}(\z_{t-1})+ (\x - \z_{t-1})^{\top}\nabla c_{t-1}(\z_{t-1}) +
            \frac{L}{2\eta}\|\x - \z_{t-1}\|_2^2)
            \right\}. \label{eqn:update-x}
        \end{align}
Notice that
\begin{align}\label{eqn:upc}
c_t(\x)&\leq c_t(\z_{t-1}) + (\x-\z_{t-1})^{\top}\nabla c_{t}(\z_{t-1}) + \frac{L}{2}\|\x-\z_{t-1}\|_2^2\nonumber\\
&\leq c_t(\z_{t-1}) + (\x-\z_{t-1})^{\top}\nabla c_{t}(\z_{t-1}) + \frac{L}{2\eta}\|\x-\z_{t-1}\|_2^2,
\end{align}
where the first inequality follows the smoothness condition in~(\ref{eqn:smooth}) and the second inequality follows from the fact $\eta\leq 1$. The inequality~(\ref{eqn:upc}) provides an upper bound for $c_t(\x)$ and therefore can be used as an approximation of $c_t(\x)$ for predicting $\x_t$. However, since $\nabla c_t(\z_{t-1})$ is unknown before the prediction, we use $\nabla c_{t-1}(\z_{t-1})$ as a surrogate for $\nabla c_t(\z_{t-1})$, leading to the updating rule in (\ref{eqn:update-x}). It is this approximation that leads to the variation  bound.  The following theorem states the  regret bound of Algorithm~\ref{alg:2}.
\begin{theorem} \label{thm:1}
Let $c_t(\cdot), t=1, \ldots, T$ be a sequence of  convex
functions with $L$-Lipschitz continuous gradients.  By setting $\eta = \min\left\{1, L/\sqrt{\text{EVAR}^s_T} \right\}$, we have the
following regret bound for
Algorithm~\ref{alg:2}
\[
\sum_{t=1}^Tc_t(\x_t) -\min_{\x\in\P}\sum_{t=1}^Tc_t(\x) \leq \max\left(L, \sqrt{\text{EVAR}^s_T}\right).
\]
\end{theorem}
\textbf{Remark:}  Comparing with the variation bound in~(\ref{eqn:var}) for the FTRL algorithm, term $L$ plays the same role as $\text{VAR}^1_T$ that accounts for the smoothness of cost functions, and term $\text{EVAR}^s_T$ plays the same role as $\text{VAR}^2_T$ that accounts for the variation in the cost functions. Compared to the FTRL algorithm, the key advantage of the improved FTRL algorithm is that the regret bound is reduced to a constant  when the cost functions change only by a constant number of times along the horizon. Of course, the extended variation $\text{EVAR}^s_T$ may not be known apriori for setting the optimal $\eta$, we can apply the standard halving tricks~\citep{Cesa-Bianchi:2006:PLG:1137817} to obtain the same order of regret bound.
To prove Theorem~\ref{thm:1}, we first present the following lemma.
\begin{lemma}\label{lem:1}
Let $c_t(\cdot), t=1, \ldots, T$ be a sequence of convex
functions with $L$-Lipschitz continuous gradients. By running Algorithm~\ref{alg:2} over $T$ trials, we have
\begin{align*}
\sum_{t=1}^T c_t(\x_t) &\leq \min\limits_{\x \in \P}
\left[\frac{L}{2\eta}\|\x\|_2^2 + \sum_{t=1}^T c_t(\z_{t-1}) + (\x -
\z_{t-1})^{\top}\nabla c_t(\z_{t-1})\right]\\
& \hspace*{0.2in}+ \frac{\eta}{2L}
\sum_{t=0}^{T-1} \|\nabla c_{t+1}(\z_t) - \nabla c_{t}(\z_t)\|_2^2.
\end{align*}
\end{lemma}
\begin{proof}
We prove the inequality by induction. When $T = 1$, we have $\x_1 =
\z_0 = 0$ and
\begin{align*}
&
\min\limits_{\x \in \P} \left[\frac{L}{2\eta}\|\x\|_2^2 +c_1(\z_0)+ (\x -
\z_0)^{\top}\nabla c_1(\z_0) \right] + \frac{\eta}{2L}\|\nabla c_1(\z_0)\|_2^2 \\
& \geq  c_1(\z_0) + \frac{\eta}{2L}\|\nabla c_1(\z_0)\|_2^2 +
\min\limits_{\x} \left\{\frac{L}{2\eta}\|\x\|_2^2 + (\x -
\z_0)^{\top}\nabla c_1(\z_0) \right\} = c_1(\z_0) = c_1(\x_1).
\end{align*}
We assume the inequality holds for $t$ and aim to prove it for $t+1$. To this end, we define
\begin{align*}
\psi_t(\x) &=  \left[\frac{L}{2\eta}\|\x\|_2^2 +
\sum_{\tau=1}^{t} c_\tau(\z_{\tau-1}) + (\x - \z_{\tau-1})^{\top}\nabla
c_\tau(\z_{\tau-1})\right] + \frac{\eta}{2L} \sum_{\tau=0}^{t - 1} \|\nabla
c_{\tau+1}(\z_\tau) - \nabla c_{\tau}(\z_\tau)\|_2^2.
\end{align*}
According to the updating procedure for $\z_t$ in step 6, we have $\z_{t}=\arg\min_{\x\in\P}\psi_{t}(\x)$. Define $\phi_t = \psi_t(\z_t) = \min_{\x\in\P}\psi_t(\x)$. Since $\psi_t(\x)$ is a $(L/\eta)$-strongly convex function, we can have
\begin{align*}
\psi_{t+1}(\x)-\psi_{t+1}(\z_{t})&\geq \frac{L}{2\eta}\|\x-\z_{t}\|_2^2 + (\x-\z_t)^{\top} \nabla\psi_{t+1}(\z_t)\\
&=\frac{L}{2\eta}\|\x-\z_{t}\|_2^2 + (\x-\z_t)^{\top} \left(\nabla\psi_{t}(\z_t)+\nabla c_{t+1}(\z_t)\right).
\end{align*}
Setting $\x=\z_{t+1}=\arg\min_{\x\in\P}\psi_{t+1}(\x)$ in the above inequality results in
\begin{align*}
\psi_{t+1}(\z_{t+1}) - \psi_{t+1}(\z_t) &= \phi_{t+1} - (\phi_t + c_{t+1}(\z_t)  + \frac{\eta}{2L}\|\nabla c_{t+1}(\z_t)
- \nabla c_t(\z_t)\|_2^2)\\
&\geq \frac{L}{2\eta}\|\z_{t+1}-\z_{t}\|_2^2 + (\z_{t+1}-\z_t)^{\top} \left(\nabla\psi_{t}(\z_t)+\nabla c_{t+1}(\z_t)\right)\\
&\geq \frac{L}{2\eta}\|\z_{t+1}-\z_{t}\|_2^2 + (\z_{t+1}-\z_t)^{\top}\nabla c_{t+1}(\z_t),
\end{align*}
where the second inequality follows from the fact $\z_t=\arg\min_{\x\in\P}\psi_t(\x)$, and therefore $(\x-\z_t)^{\top}\nabla\psi_t(\z_t)\geq 0, \forall\x\in\P$. Then we have
\begin{align}
\lefteqn{\phi_{t+1} - \phi_t - \frac{\eta}{2L}\|\nabla c_{t+1}(\z_t)
- \nabla c_t(\z_t)\|_2^2} \label{eqn:ind}\\
& \geq  \min\limits_{\x \in \P} \left\{ \frac{L}{2\eta}\|\x -\z_t\|_2^2 + (\x - \z_{t})^{\top}\nabla c_{t+1}(\z_t)+ c_{t+1}(\z_t)\right\} \nonumber\\
& =  \min\limits_{\x\in \P}\left\{  \underbrace{\frac{L}{2\eta}\|\x - \z_t\|_2^2 + (\x - \z_{t})^{\top}\nabla c_{t}(\z_t)}\limits_{\rho(\x)} + c_{t+1}(\z_t)+ \underbrace{(\x -
\z_{t})^{\top}(\nabla c_{t+1}(\z_t) - \nabla c_{t}(\z_t))}\limits_{r(\x)}\right\}.\nonumber
\end{align}
To bound the right hand side, we note that $\x_{t+1}$ is the minimizer of $\rho(\x)$ by step 4 in Algorithm~\ref{alg:2}, and $\rho(\x)$ is a $L/\eta$-strongly convex function, so we have
\begin{align*}
\rho(\x)\geq \rho(\x_{t+1}) + \underbrace{(\x-\x_{t+1})^{\top}\nabla\rho(\x_{t+1})}\limits_{\geq 0}+ \frac{L}{2\eta}\|\x-\x_{t+1}\|_2^2\geq \rho(\x_{t+1}) +  \frac{L}{2\eta}\|\x-\x_{t+1}\|_2^2.
\end{align*}
Then we have
\begin{align*}
\rho(\x) + c_{t+1}(\z_t) + r(\x) \geq  \rho(\x_{t+1})  + c_{t+1}(\z_t) +  \frac{L}{2\eta}\|\x-\x_{t+1}\|_2^2+ r(\x).
\end{align*}
We  proceed by bounding~(\ref{eqn:ind}) as
\begin{align*}
\phi_{t+1} &- \phi_t - \frac{\eta}{2L}\|\nabla c_{t+1}(\z_t)
- \nabla c_t(\z_t)\|_2^2\\
\geq&  \frac{L}{2\eta}\|\x_{t+1} - \z_t\|^2_2 +
(\x_{t+1} - \z_t)^{\top}\nabla c_t(\z_t)+ c_{t+1}(\z_t)  \\
&+ \min\limits_{\x\in \P}\left\{ \frac{L}{2\eta}\|\x -
\x_{t+1}\|_2^2 + (\x - \z_t)^{\top}(\nabla c_{t+1}(\z_t) - \nabla
c_{t}(\z_t))\right\} \\
=&  \frac{L}{2\eta}\|\x_{t+1} - \z_t\|^2_2  +(\x_{t+1} - \z_t)^{\top}\nabla c_{t+1}(\z_t)  + c_{t+1}(\z_t) \\
&  + \min\limits_{\x\in \P}\left\{ \frac{L}{2\eta}\|\x -\x_{t+1}\|_2^2 + (\x - \x_{t+1})^{\top}(\nabla c_{t+1}(\z_t) - \nabla c_{t}(\z_t))\right\} \\
\geq&  \frac{L}{2\eta}\|\x_{t+1} - \z_t\|^2_2  +(\x_{t+1} - \z_t)^{\top}\nabla c_{t+1}(\z_t)+ c_{t+1}(\z_t) \\
&+ \min\limits_{\x}\left\{ \frac{L}{2\eta}\|\x -\x_{t+1}\|_2^2 + (\x - \x_{t+1})^{\top}(\nabla c_{t+1}(\z_t) - \nabla c_{t}(\z_t))\right\} \\
=&  \frac{L}{2\eta}\|\x_{t+1} - \z_t\|^2_2 +
(\x_{t+1} - \z_t)^{\top}\nabla c_{t+1}(\z_t) + c_{t+1}(\z_t) -
\frac{\eta}{2L}\|\nabla
c_{t+1}(\z_t) - \nabla c_{t}(\z_t)\|_2^2 \\
\geq&  c_{t+1}(\x_{t+1}) - \frac{\eta}{2L}\|\nabla c_{t+1}(\z_t) -
\nabla c_{t}(\z_t)\|_2^2,
\end{align*}
where the first equality follows by writing
$(\x_{t+1} - \z_t)^{\top}\nabla c_t(\z_t)=
(\x_{t+1} - \z_t)^{\top}\nabla c_{t+1}(\z_t) -
(\x_{t+1} - \z_t)^{\top}(\nabla c_{t+1}(\z_t)-\nabla c_t(\z_t))$, and the last inequality follows from the smoothness condition of $c_{t+1}(\x)$.
Since $\phi_t \geq \sum_{\tau=1}^t c_{\tau}(\x_\tau)$, we have
$\phi_{t+1} \geq \sum_{\tau=1}^{t+1} c_\tau(\x_\tau)$.
\end{proof}

\begin{proof}[of Theorem~\ref{thm:1}]
By $\|\x\|_2\leq 1,\forall \x\in\P \subseteq \mathcal B$,  and the convexity of $c_t(\x)$, we have
 \[
 \min\limits_{\x \in \P}
\left\{\frac{L}{2\eta}\|\x\|_2^2 + \sum_{t=1}^T c_t(\z_{t-1}) + (\x -
\z_{t-1})^{\top}\nabla c_t(\z_{t-1})\right\}\leq \frac{L}{2\eta}+\min_{\x\in\P}\sum_{t=1}^Tc_t(\x).
 \]
 Combining the above result with Lemma~\ref{lem:1}, we have
 \begin{align*}
 \sum_{t=1}^Tc_t(\x_t)-\min_{\x\in\P}\sum_{t=1}^T c_t(\x)\leq \frac{L}{2\eta} + \frac{\eta}{2L}\text{EVAR}^s_T.
 \end{align*}
 By choosing $\eta=\min(1, L/\sqrt{\text{EVAR}^s_T})$, we have the regret bound in Theorem~\ref{thm:1}.
\end{proof}
\subsection{A Prox Method for Online Convex Optimization}
\label{sec:PM}
In this subsection, we present a prox method for online convex optimization that shares the same order of regret bound as the improved FTRL algorithm. It is closely related to the prox method in~\citep{Nemirovski2005} by maintaining two sets of vectors $\x_{1:T}$ and $\z_{1:T}$, where $\x_t$ and $\z_t$ are computed by gradient mappings using $\nabla c_{t-1}(\z_{t-1})$, and $\nabla c_t(\x_t)$, respectively, as
\begin{algorithm}[t]
\center \caption{Prox Method for
Online Convex Optimization}
\begin{algorithmic}[1] \label{alg:3}
    \STATE {\bf Input}: $\eta >0$

    \STATE {\bf Initialization}: $\z_0 = \mathbf{0}$ and $c_0(\x) =
    0$

    \FOR{$t = 1, \ldots, T$}
        \STATE Predict $\x_t$ by
        $
            \x_t = \mathop{\arg\min}\limits_{\x \in \P} \left\{\x^{\top}\nabla c_{t-1}(\z_{t-1})+
            \frac{L}{2\eta}\|\x - \z_{t-1}\|_2^2
            \right\}
        $
        \STATE Receive a cost function $c_t(\cdot)$ and incur a loss $c_t(\x_t)$
       \STATE Update $\z_t$ by
        $
            \z_{t} =\mathop{\arg\min}\limits_{\x \in \P} \left\{\x^{\top}\nabla c_{t}(\x_{t})+\frac{L}{2\eta}\|\x-\z_{t-1}\|_2^2 \right\}
        $
    \ENDFOR
\end{algorithmic}
\end{algorithm}
presented in Algorithm~\ref{alg:3}.  Algorithm~\ref{alg:3} only differs from Algorithm~\ref{alg:2} in updating the searching points $\z_t$. Algorithm~\ref{alg:2} updates $\z_t$ by the FTRL scheme using \textit{all} the gradients of the cost functions at $\{\z_{\tau}\}_{\tau = 1}^{t - 1}$, while Algorithm~\ref{alg:3} updates $\z_t$ by a prox method using a \textit{single} gradient $\nabla c_t(\x_t)$. It is this difference that makes it easier to extend the prox method to a bandit setting, which will be discussed in section~\ref{sec:conc}.  The following theorem states the regret bound of the prox method for online convex optimization.
\begin{theorem} \label{thm:2}
Let $c_t(\cdot), t=1, \ldots, T$ be a sequence of convex
functions with L-Lipschitz continuous gradients.  By setting
$\eta = (1/2)\min\left\{1, L/\sqrt{\text{EVAR}^s_T} \right\}$, we have the
following regret bound for Algorithm~\ref{alg:3}
\[
\sum_{t=1}^Tc_t(\x_t)-\min_{\x\in\P}\sum_{t=1}^Tc_t(\x) \leq 2\max\left(L, \sqrt{\text{EVAR}^s_T}\right).
\]
\end{theorem}
Compared to Theorem~\ref{thm:1}, the regret bound in Theorem~\ref{thm:2} is slightly worse by a factor of $2$.
To prove Theorem~\ref{thm:2}, we need the following lemma, which is the Lemma 3.1 in~\citep{Nemirovski2005} stated in our notations.

 \begin{lemma}[\textbf{Lemma 3.1}~\citep{Nemirovski2005}]\label{lem:6}
Let $\omega(\z)$ be a  $\alpha$-strongly convex function with respect to the norm $\|\cdot\|$, whose dual norm is denoted by $\|\cdot\|_*$,  and $D(\x,\z) = \omega(\x)- (\omega(\z) + (\x-\z)^{\top}\omega'(\z))$ be the Bregman distance induced by function $\omega(\x)$. Let $Z$ be a convex compact set, and $U\subseteq Z$ be convex and closed.  Let $\z\in Z$, $\gamma>0$, Consider the points,
 \begin{align}
 \x &= \arg\min_{\u\in U} \gamma\u^{\top}\xi + D(\u, \z)\label{eqn:project1},\\
 \z_+&=\arg\min_{\u\in U} \gamma\u^{\top}\zeta + D(\u,\z),\label{eqn:project2}
 \end{align}
 then for any $\u\in U$, we have
 \begin{align}\label{eqn:ineq}
 \gamma\zeta^{\top}(\x-\u)\leq  D(\u,\z) - D(\u, \z_+) + \frac{\gamma^2}{\alpha}\|\xi-\zeta\|_*^2 - \frac{\alpha}{2}[\|\x-\z\|^2 + \|\x-\z_+\|^2].
 \end{align}
 \end{lemma}


In order not to put readers in struggling with complex notations in~\citep{Nemirovski2005} for the proof of Lemma~\ref{lem:6}, we present a detailed proof  in Appendix A which is an adaption of the original proof to our notations. \\

\begin{proof}[of Theorem~\ref{thm:2}]
First, we note that the two updates in step 4 and step 6 of Algorithm~\ref{alg:3} fit in the Lemma~\ref{lem:6}  if we let $U=Z=\P$,  $\z=\z_{t-1}$, $\x=\x_t$, $\z_+=\z_t$,  and $\omega(\x)=\frac{1}{2}\|\x\|_2^2$, which is $1$-strongly convex function with respect to $\|\cdot\|_2$.  Then $D(\u,\z)=\frac{1}{2}\|\u-\z\|_2^2$. As a result, the two updates for $\x_t, \z_t$ in Algorithm~\ref{alg:3} are exactly the updates in~(\ref{eqn:project1}) and (\ref{eqn:project2}) with $\z=\z_{t-1}, \gamma=\eta/L$, $\xi=\nabla c_{t-1}(\z_{t-1})$, and $\zeta=\nabla c_t(\x_t)$.  Replacing these into~(\ref{eqn:ineq}), we have the following inequality, 
\begin{align*}
\frac{\eta}{L} (\x_t - \z)^{\top}\nabla c_t(\x_t) &\le
\frac{1}{2}\left(\|\z-\z_{t-1} \|_2^2 - \|\z-\z_{t} \|_2^2\right)\\
& \hspace*{0.1in}+
\frac{\eta^2}{L^2}\|\nabla c_{t}(\x_t) - \nabla c_{t-1}(\z_{t-1})\|_2^2 -
\frac{1}{2}\|\x_t - \z_{t-1}\|_2^2.
\end{align*}
Then we have
\begin{align*}
&\frac{\eta}{L}(c_t(\x_t) - c_t(\z))\leq \frac{\eta}{L} (\x_t - \z)^{\top}\nabla c_t(\x_t) \leq  \frac{1}{2}\left(\|\z-\z_{t-1} \|_2^2 - \|\z-\z_{t} \|_2^2\right)\\
&+ \frac{2\eta^2}{L^2} \|\nabla
c_{t}(\z_{t-1}) - \nabla c_{t-1}(\z_{t-1})\|_2^2+ \frac{2\eta^2}{L^2}\|\nabla
c_{t}(\x_t) - \nabla c_{t}(\z_{t-1})\|_2^2  - \frac{1}{2}\|\x_t - \z_{t-1}\|_2^2  \\
& \leq  \frac{1}{2}\left(\|\z-\z_{t-1}\|_2^2 - \|\z-\z_{t}\|_2^2\right)+ \frac{2\eta^2}{L^2}\|\nabla c_{t}(\z_{t-1}) - \nabla c_{t-1}(\z_{t-1})\|_2^2 + \underbrace{\left(2\eta^2 - \frac{1}{2}\right)\|\x_t - \z_{t-1}\|_2^2}\limits_{\leq 0\text{ due to $\eta\leq 1/2$}},
\end{align*}
where the first inequality follows the convexity of $c_t(\x)$, and the third inequality follows the smoothness  of $c_t(\x)$.  
By taking the summation over $t=1,\cdots, T$  with  $\z^*=\arg\min\limits_{\z\in\P}\sum_{t=1}^Tc_t(\z)$, and dividing both sides by $\eta/L$, 
we have
\begin{eqnarray*}
\sum_{t=1}^Tc_t(\x_t) - \min_{\x\in\P}\sum_{t=1}c_t(\x) & \leq & \frac{L}{2\eta}+ \frac{2\eta}{L}\sum_{t=0}^{T-1}\|\nabla c_{t+1}(\z_{t}) - \nabla c_{t}(\z_{t})\|_2^2.
\end{eqnarray*}
We complete the proof by plugging the value of $\eta$.
\end{proof}

\noindent \textbf{Remark:} Note that the prox method, together with Lemma~\ref{lem:6}  provides an easy way to generalize the framework based on Euclidean norm to a general norm. To be precise, let $\|\cdot\|$ denote a general norm, $\|\cdot\|_*$ denote its dual norm, $\omega(\z)$ be a  $\alpha$-strongly convex function with respect to the norm $\|\cdot\|$,  and $D(\x,\z) = \omega(\x)- (\omega(\z) + (\x-\z)^{\top}\omega'(\z))$ be the Bregman distance induced by function $\omega(\x)$. 
 Let $c_t(\cdot), t=1,\cdots, T$ be $L$-smooth functions with respect to norm $\|\cdot\|$, i.e.,
$\|\nabla c_t(\x) - \nabla c_t(\z)\|_* \leq L \|\x- \z\|$.
Correspondingly, we define the extended sequential variation based on the general norm as follows:
\begin{align}\label{eqn:ge}
\text{EVAR}^{gs}_T &= \sum_{t=0}^{T-1}\|\nabla c_{t+1}(\z_t) - \nabla c_t(\z_t)\|_*^2.
\end{align}

Algorithm~\ref{alg:3-1} gives the detailed steps for the general framework. 
We note that the key differences from Algorithm~\ref{alg:3} are:  $\z_0$ is set to $\min_{\z\in\P}\omega(\z)$, and the Euclidean distances in steps 4 and 6 are replaced by Bregman distances, i.e.,  
 \begin{align*}
           &\displaystyle \x_t = \mathop{\arg\min}\limits_{\x \in \P} \left\{\x^{\top}\nabla c_{t-1}(\z_{t-1})+
           \frac{L}{\eta}D(\x, \z_{t-1}) \right\},\\
           &\displaystyle \z_{t} =\mathop{\arg\min}\limits_{\x \in \P} \left\{\x^{\top}\nabla c_{t}(\x_{t})+ \frac{L}{\eta}D(\x, \z_{t-1})\right\}.
\end{align*}
The following theorem states the variation-based regret bound for the general norm  framework, where $R$ measure the size of $\P$ defined as $R = \sqrt{2(\max_{\x\in\P} \omega(\x) - \min_{\x\in\P}\omega(\x))}$. 
\begin{theorem} \label{thm:2-1}
Let $c_t(\cdot), t=1, \ldots, T$ be a sequence of convex
functions whose gradients are L-Lipschitz continuous, $\omega(\z)$ be a $\alpha$-strongly convex function, both with respect to norm $\|\cdot\|$, and $\text{EVAR}^{gs}_T$ be defined in~(\ref{eqn:ge}). By setting
$\eta = (1/2)\min\left\{\sqrt{\alpha}, LR/\sqrt{\text{EVAR}^{gs}_T} \right\}$, we have the
following regret bound
\[
\sum_{t=1}^Tc_t(\x_t)-\min_{\x\in\P}\sum_{t=1}^Tc_t(\x) \leq 2R\max\left(LR/\sqrt{\alpha}, \sqrt{\text{EVAR}^{gs}_T}\right).
\]
\end{theorem}
We skip the proof since it is similar to that of Theorem~\ref{thm:2}.


\begin{algorithm}[t]
\center \caption{General Prox Method for
Online Convex Optimization}
\begin{algorithmic}[1] \label{alg:3-1}
    \STATE {\bf Input}: $\eta >0, \omega(\z)$

    \STATE {\bf Initialization}: $\z_0 = \min_{\z\in\P}\omega(\z)$ and $c_0(\x) =
    0$

    \FOR{$t = 1, \ldots, T$}
        \STATE Predict $\x_t$ by
        $
           \displaystyle \x_t = \mathop{\arg\min}\limits_{\x \in \P} \left\{\x^{\top}\nabla c_{t-1}(\z_{t-1})+
           \frac{L}{\eta}D(\x, \z_{t-1})
            \right\}
        $
        \STATE Receive a cost function $c_t(\cdot)$ and incur a loss $c_t(\x_t)$
       \STATE Update $\z_t$ by
        $
           \displaystyle \z_{t} =\mathop{\arg\min}\limits_{\x \in \P} \left\{\x^{\top}\nabla c_{t}(\x_{t})+ \frac{L}{\eta}D(\x, \z_{t-1})\right\}
        $
    \ENDFOR
\end{algorithmic}
\end{algorithm}

\section{Conclusions and Open Problems}\label{sec:conc}
In this paper, we proposed two algorithms for online convex optimization that bound the regret by the variation of cost functions. The first algorithm is an improvement of FTRL algorithm, and the second algorithm is based on the prox method.

One open problem is how to extend the proposed algorithms to the case where the learner only receives partial feedback about the cost functions. One common scenario of partial feedback is that the learner only receives the cost $c(\x_t)$ at the predicted point $\x_t$ but without observing the entire cost function $c_t(\x)$. This setup is usually referred as  bandit setting, and the related online learning problem is called online bandit convex optimization. Many algorithms have been proposed for online bandit convex optimization with regret bounds stated in number of trials  \citep{flaxman-2005-online, Awerbuch:2004:ARE:1007352.1007367,DBLP:conf/soda/DaniH06, DBLP:conf/colt/AbernethyHR08}. In~\citep{DBLP:conf/soda/HazanK09}, the authors extended the FTRL algorithm to online bandit \textit{linear} optimization and obtained a variation-based regret bound of $O(poly(d)\sqrt{\text{VAR}_T\log(T)}+poly(d\log(T)))$, where $\text{VAR}_T$ is the total variation of the cost vectors. The open question is how to develop algorithms for general  online bandit convex optimization with a variation-based regret bound. Directly extending the proposed algorithms to the bandit setting may be difficult because they need to keep track of and update two sets of solutions $\x_{1:T}$ and $\z_{1:T}$, and therefore it is insufficient to query each cost function only once. One possibility is to explore the multi-point bandit setting proposed in~\citep{DBLP:conf/colt/AgarwalDX10}, where multiple points can be queried for each cost function. In Appendix B, we extend the prox method to the multi-point bandit setting using $O(d)$ queries, and prove a variation-based regret bound which is optimal when the variation of cost functions is independent from $T$. It remains as an open problem how to achieve a variation-based regret bound with a constant number of queries independent from the dimension $d$. Another open problem for the
future work is how to reduce the dependence on $T$ in the regret bound for online bandit convex
optimization.








\bibliography{online-convex}

\section*{Appendix A: Proof of Lemma \ref{lem:6}}
By using the definition of Bregman distance $D(\u, \z)$, we can write equations~(\ref{eqn:project1}) and~(\ref{eqn:project2}) as
\begin{align*}
 \x &= \arg\min_{\u\in U} \u^{\top}(\gamma\xi - \omega'(\z)) + \omega(\u),\\
 \z_+&=\arg\min_{\u\in U} \u^{\top}(\gamma\zeta-\omega'(\z)) + \omega(\u),
 \end{align*}
 by the first oder optimality  condition, we have
 \begin{align}
& (\u-\x)^{\top} (\gamma\xi - \omega'(\z) + \omega'(\x))\geq 0, \forall \u\in U,\label{eqn:b1}\\
&(\u-\z_+)^{\top} (\gamma\zeta- \omega'(\z) + \omega'(\z_+))\geq 0, \forall \u\in U.\label{eqn:b2}
 \end{align}
 Applying~(\ref{eqn:b1}) with $\u= \z_+$ and~(\ref{eqn:b2}) with $\u= \x$, we get
 \begin{align*}
 & \gamma(\x-\z_+)^{\top}\xi\leq  (\omega'(\z) - \omega'(\x))^{\top}(\x-\z_+),\\
&\gamma(\z_+-\x)^{\top} \zeta\leq (\omega'(\z) - \omega'(\z_+))^{\top}(\z_+-\x).
 \end{align*}
 Summing up the two inequalities, we have
 \begin{align*}
 \gamma (\x-\z_+)^{\top}(\xi-\zeta) \leq (\omega'(\z_+) - \omega'(\x))^{\top}(\x-\z_+).
 \end{align*}
 Then
 \begin{align}
 \gamma \|\xi-\zeta\|_*\|\x-\z_+\|&\geq   -\gamma (\x-\z_+)^{\top}(\xi-\zeta)\geq (\omega'(\z_+) - \omega'(\x))^{\top}(\z_+-\x)\nonumber\\
 &\geq\alpha\|\z_+-\x\|^2.\label{eqn:b3}
 \end{align}
 where in the last inequality, we use the strong convexity of $\omega(\x)$.

 \begin{align*}
D(\u, \z)&- D(\u, \z_+)=\omega(\z_+)-\omega(\z) +(\u-\z_+)^{\top}  \omega'(\z_+) -(\u-\z)^{\top}\omega'(\z) \\
= &\omega(\z_+)-\omega(\z) +(\u-\z_+)^{\top}\omega'(\z_+) -( \u - \z_+)^{\top}\omega'(\z) -  (\z_+ -\z)^{\top} \omega'(\z) \\
= & \omega(\z_+) -\omega(\z)-(\z_+ -\z )^{\top} \omega'(\z) +(\u-\z_+)^{\top} (\omega'(\z_+)- \omega'(\z))\\
=& \omega(\z_+) -\omega(\z)  -( \z_+ -\z )^{\top} \omega'(\z) + (\u-\z_+)^{\top}(\gamma\zeta+   \omega'(\z_+)- \omega'(\z)) - (\u-\z_+)^{\top} \gamma\zeta\\
\geq & \omega(\z_+) -\omega(\z)  -  ( \z_+ -\z)^{\top}  \omega'(\z)  - (\u-\z_+)^{\top}\gamma \zeta \\
=& \underbrace{\omega(\z_+)  -\omega(\z) - (\z_+ -\z)^{\top}  \omega'(\z)  -  ( \x-\z_+)^{\top} \gamma\zeta}\limits_{\epsilon} + (\x -\u)^{\top}\gamma\zeta,
\end{align*}
where the inequality follows from~(\ref{eqn:b2}). We proceed by bounding $\epsilon$ as:
 \begin{align*}
 \epsilon =& \omega(\z_+)   -\omega(\z)- (\z_+ -\z ) ^{\top}\omega'(\z) - (\x-\z_+)^{\top}\gamma\zeta\\
=& \omega(\z_+)  -\omega(\z) - ( \z_+ -\z )^{\top} \omega'(\z) - (\x-\z_+)^{\top}\gamma ( \zeta - \xi)-  ( \x-\z_+)^{\top}\gamma \xi \\
=&\omega(\z_+)  -\omega(\z) - ( \z_+ -\z )^{\top} \omega'(\z) - (\x-\z_+)^{\top}\gamma ( \zeta - \xi)\\
& + (\z_+ - \x)^{\top}(\gamma  \xi - \omega'(\z) + \omega'(\x) )- (\z_+- \x)^{\top}(\omega'(\x)-\omega'(\z) )\\
\geq & \omega(\z_+)  -\omega(\z) - ( \z_+ -\z )^{\top} \omega'(\z) - (\x-\z_+)^{\top}\gamma ( \zeta - \xi)-( \z_+- \x)^{\top} (\omega'(\x)- \omega'(\z))\\
= & \omega(\z_+) -\omega(\z) - (\x -\z)^{\top} \omega'(\z) -(\x-\z_+)^{\top}\gamma (\zeta - \xi) - ( \z_+- \x)^{\top} \omega'(\x) \\
= &\left[ \omega(\z_+)- \omega(\x) - (\z_+- \x)^{\top}\omega'(\x) \right]+\left[\omega(\x) -\omega(\z) -(\x -\z )^{\top} \omega'(\z) \right]- (\x-\z_+,)^{\top}\gamma ( \zeta- \xi)\\
\geq & \frac{\alpha}{2} \| \x - \z_+\|^2 + \frac{\alpha}{2}\|\x - \z \|^2 - \gamma \| \x-\z_+ \| \| \zeta - \xi \|_*\\
\geq & \frac{\alpha}{2}\{ \| \x-\z_+ \|^2 + \|\x - \z \|^2\} -\frac{\gamma^2}{\alpha} \|\zeta - \xi \|_*^2,
\end{align*}
where the first inequality follows from~(\ref{eqn:b1}), the second inequality follows from the strong convexity of $\omega(\x)$, and the last inequality follows from~(\ref{eqn:b3}). Combining the above results, we have
\begin{align*}
\gamma(\x -\u)^{\top}\zeta \leq D(\u, \z)&- D(\u, \z_+) + \frac{\gamma^2}{\alpha} \|\zeta - \xi \|_*^2 - \frac{\alpha}{2} \{\| \x-\z_+  \|^2 + \|\x - \z \|^2\} .
\end{align*}

\section*{Appendix B: A Randomized Algorithm for Online Bandit Convex Optimization}
In this appendix, we present a randomized algorithm for online bandit convex optimization with a variation-based regret bound.  Besides the smoothness assumption of the cost functions, and the boundness assumption about the domain $\P\subseteq\mathcal B$,  we further assume that (i) there exists $r\leq 1$ such that $
r\mathcal B\subseteq \P$, and (ii) the cost function themselves are Lipschitz continuous, i.e., there exists a constant $G$ such that
$|c_t(\x)-c_t(\z)|\leq G\|\x-\z\|_2, \forall \x, \z\in\P, \forall t$. 
To present the algorithm, we introduce a few notations. Let $i_t$ denote a random index in $\{1,\cdots, d\}$, and  
\begin{align*}
g_{t-1}(\z_{t-1})&= \frac{1}{\delta}\sum_{i=1}^d \left(c_{t-1}(\z_{t-1}+\delta \e_i)- c_{t-1}(\z_{t-1})\right)\e_i\\
\widehat g_t(\x_t, \e_{i_t})&=\frac{d}{\delta} (c_t(\x_t+\delta \e_{i_t})-c_t(\x_t))\e_{i_t}\\
\tilde g_t(\x_t, \e_{i_t})&= \widehat g_t(\x_t, \e_{i_t}) + g_{t-1}(\z_{t-1}) - \widehat g_{t-1}(\z_{t-1}, \e_{i_t})
\end{align*} 
The detailed steps are  shown in Algorithm~\ref{alg:5}.  
We use notation $\tilde g_t(\x_t) = \tilde g_t(\x_t, \e_{i_t})$ for short. It can be shown that $\E_{t}[\tilde g_t(\x_t)]=\E_{t}[\widehat g_t(\x_t,\e_{i_t})]$. The reason to use $\tilde g_t(\x_t)$ rather than $\widehat g_t(\x_t, \e_{i_t})$  in updating $\z_t$ is to cancel $g_{t-1}(\z_{t-1})$ in updating $\x_t$.  To prove the regret bound, we define another variation of cost functions by
\begin{align}
\text{EVAR}^{cs}_T=\sum_{t=0}^{T-1} \max_{\x\in\P}|c_{t+1}(\x)-c_t(\x)| \label{eqn:var-2}
\end{align}
Unlike the variation defined in~(\ref{eqn:vard}) that uses the gradient of the cost functions, the variation in (\ref{eqn:var-2}) is defined according to the values of cost functions. The reason why we bound the regret of Algorithm~\ref{alg:5} by the variation defined in~(\ref{eqn:var-2})  by the values of the cost functions rather than the one defined in~(\ref{eqn:vard}) by the gradient of the cost functions is that in the bandit setting, we only have point evaluations of the cost functions. The following theorem states the regret bound for Algorithm~\ref{alg:5}. 

\begin{algorithm}[t]
\center \caption{Randomized Online Bandit Convex Optimization}
\begin{algorithmic}[1] \label{alg:5}
    \STATE {\bf Input}: $\eta$, $\alpha$, $\delta>0$

    \STATE {\bf Initialization}: $\z_0 = \mathbf{0}$ and $c_0(\x) =
    0$
    \FOR{$t = 1, \ldots, T$}
        \STATE Compute $\x_t$ by
        $
            \x_t = \mathop{\arg\min}\limits_{\x \in(1-\alpha) \P} \left\{\x^{\top}g_{t-1}(\z_{t-1})+
            \frac{G}{2\eta}\|\x - \z_{t-1}\|_2^2
            \right\}
        $
        \STATE Random sample $i_t\in\{1,\cdots, d\}$.
        \STATE Observe $c_t(\x_t), c_t(\x_t+\delta\e_{i_t})$
       \STATE Update $\z_t$ by
        $
           \z_{t} = \mathop{\arg\min}\limits_{\x \in (1-\alpha)\P} \left\{\x^{\top}\tilde g_t(\x_t)+\frac{G}{2\eta}\|\x-\z_{t-1}\|_2^2 \right\}
       $
        \STATE Observe $c_t(\z_t), c_t(\z_t+\delta\e_i),i=1,\cdots, d$
    \ENDFOR
\end{algorithmic}
\end{algorithm}
\begin{theorem} \label{thm:4}
Let $c_t(\cdot), t=1, \ldots, T$ be a sequence of $G$-Lipschitz continuous  convex
functions with $L$-Lipschitz continuous gradients. By setting $\displaystyle\delta =\sqrt{\frac{4d\max(G,\sqrt{\text{EVAR}^{cs}_T})}{(dL+G(1+1/r))T}}$, 
$\displaystyle\eta =\frac{\delta}{4d}\min\left\{1, \frac{G}{\sqrt{\text{EVAR}^{cs}_T}} \right\}$, and $\displaystyle\alpha=\frac{\delta}{r}$, we have the
regret bound for
Algorithm~\ref{alg:5} by
\begin{align*}
\E\left[\sum_{t=1}^T\frac{1}{2}(c_t(\x_t) +c_t(\x_t+\delta\e_{i_t}))\right] &-  \min\limits_{\x
\in \P} \sum_{t=1}^T c_t(\x)\\
& \leq 4\sqrt{\max\left(G, \sqrt{\text{EVAR}^{cs}_T}\right)d\left(dL+G(1+1/r)\right)T}
\end{align*}
\end{theorem}
\textbf{Remark:} Similar to the regret bound in~(\citet{DBLP:conf/colt/AgarwalDX10}, Theorem 9), Algorithm~\ref{alg:5} also gives the optimal regret bound $O(\sqrt{T})$ when the variation is independent of the number of trials. Our regret bound has a better dependence on the dimension $d$ (i.e., $d$) compared with the regret bound in~\citep{DBLP:conf/colt/AgarwalDX10} (i.e., $d^2$).
\begin{proof}
Let $h_t(\x)=c_t(\x) + (\tilde g_t(\x_t)-\nabla c_t(\x_t))^{\top}\x$.  It is easy seen that $\nabla h_t(\x_t)=\tilde g_t(\x_t)$.
Followed by Lemma \ref{lem:6}, we have for any  $\z\in(1-\alpha)\P$
\begin{align*}
&\frac{\eta}{G}\nabla h_t(\x_t)^{\top}(\x_t - \z) \leq
\frac{1}{2}\left(\|\z-\z_{t-1}\|_2^2 - \|\z-\z_{t}\|_2^2 \right) +
\frac{\eta^2}{G^2}\|\tilde g_t(\x_t) - g_{t-1}(\z_{t-1})\|_2^2
- \frac{1}{2}\|\x_t - \z_{t-1}\|_2^2 \\
& =  \frac{1}{2}\left(\|\z-\z_{t-1}\|_2^2 - \|\z-\z_{t}\|_2^2 \right) - \frac{1}{2}\|\x_t - \z_t\|_2^2 \\
&+ \frac{\eta^2}{G^2}\|\widehat g_t(\x_t,\e_{i_t}) - \widehat g_t(\z_{t-1},\e_{i_t})  + \widehat g_{t}(\z_{t-1},\e_{i_t}) -\widehat g_{t-1}(\z_{t-1}, \e_{i_t})\|_2^2
\end{align*}
By expanding the last term using the definitions of $\widehat g_t$ and the Lipschitz continuity of $c_t(\cdot)$, we have
\begin{align*}
&\frac{\eta}{G}\nabla h_t(\x_t)^{\top}(\x_t - \z) \\
& \leq  \frac{1}{2}\left(\|\z-\z_{t-1}\|_2^2 - \|\z-\z_{t}\|_2^2 \right)
- \frac{1}{2}\|\x_t - \z_t\|_2^2   + \frac{8\eta^2d^2}{\delta^2}\|\x_t-\z_{t-1}\|_2^2+ \frac{8\eta^2d^2}{\delta^2G^2}\max_{\x\in\P}|c_t(\x)-c_{t-1}(\x)|^2 \\
&\leq   \frac{1}{2}\left(\|\z-\z_{t-1}\|_2^2 - \|\z-\z_{t}\|_2^2 \right) +\left(\frac{8\eta^2d^2}{\delta^2} -
\frac{1}{2}\right)\|\x_t - \z_{t-1}\|_2^2   +  \frac{8\eta^2d^2}{G^2\delta^2}\max_{\x\in\P}|c_t(\x)-c_{t-1}(\x)|^2\\
& \leq   \frac{1}{2}\left(\|\z-\z_{t-1}\|_2^2 - \|\z-\z_{t}\|_2^2 \right)  +   \frac{8\eta^2d^2}{G^2\delta^2}\max_{\x\in\P}|c_t(\x)-c_{t-1}(\x)|^2
\end{align*}
where the last inequality follows from the fact $\eta\leq \delta/(4d)$. Taking summation over $t=1,\cdots, T$, and by convexity of $h_t(\x)$, we have
\begin{align*}
\sum_{t=1}^Th_t(\x_t) -\min_{\x\in\P}h_t((1-\alpha)\x)&\leq \frac{G}{2\eta} + \frac{8\eta d^2}{G\delta^2}\text{EVAR}^{cs}_T\leq \frac{4d}{\delta}\max\left(G, \sqrt{\text{EVAR}^{cs}_T}\right)
\end{align*}
Following the the proof of Theorem 8 in~\citep{DBLP:conf/colt/AgarwalDX10},
we have
\begin{align*}
&\E\left[\sum_{t=1}c_t(\x_t) -\sum_{t=1}^T c_t(\x)\right] \leq \E\left[\sum_{t=1}^T h_t(\x_t) -\sum_{t=1}^T h_t(\x)\right] + \E\left[\sum_{t=1}^T c_t(\x_t) - h_t(\x_t) - c_t(\x) + h_t(\x)\right]\\
 &\leq \E\left[\sum_{t=1}^T h_t(\x_t) -\sum_{t=1}^T h_t(\x)\right] + \E\left[\sum_{t=1}^T (\E_t[\tilde g_t(\x_t)]- \nabla c_t(\x_t))^{\top}(\x-\x_t))\right]\\
 &\leq \E\left[\sum_{t=1}^T h_t(\x_t) -\sum_{t=1}^T h_t(\x)\right] +  dL\delta T
\end{align*}
where the last inequality follows from $\|\x-\x_t\|\leq 2
$, $\E_t[\tilde g_t(\x_t)] =\E_t[\widehat g_t(\x_t, \e_{i_t})]$ and the following inequality
~\citep{DBLP:conf/colt/AgarwalDX10}. 
\begin{align*}
\|\E_t[\widehat g_t(\x_t,\e_{i_t})] - \nabla c_t(\x_t)]\|_2\leq \frac{dL\delta}{2}
\end{align*}
Then we have
\begin{align*}
\E\left[\sum_{t=1}^T\frac{1}{2}\left(c_t(\x_t)  + c_t(\x_t+\delta\e_{i_t})\right)\right]-\min\limits_{\x \in \P}
\left[\sum_{t=1}^T c_t(\x)\right]& \leq  \frac{4d}{\delta}\max\left(G, \sqrt{\text{EVAR}^{cs}_T}\right)\\
&+ \delta dLT+ \delta G T
+ \alpha GT
\end{align*}
Plugging the stated values of $\delta$ and $\alpha$ completes the proof.
\end{proof}\\


\end{document}